\documentclass[11pt]{article}

\usepackage[margin=1in]{geometry}
\usepackage{amsmath,amssymb,amsfonts}
\usepackage{graphicx}
\usepackage{microtype}
\usepackage{authblk}   

\usepackage[utf8]{inputenc} 
\usepackage[T1]{fontenc}    
\usepackage{hyperref}       
\usepackage{url}            
\usepackage{booktabs}       
\usepackage{amsfonts}       
\usepackage{nicefrac}       
\usepackage{microtype}      
\usepackage{xcolor}         
\usepackage{macro}
\usepackage{enumitem}

\newcommand{\xri}[3][]{%
\if\relax\detokenize{#3}\relax
#1_{r_{#2}}%
\else
#1_{r_{#2 \to #3}}%
\fi
}

\newcommand\nnfootnote[1]{%
  \begin{NoHyper}
  \renewcommand\thefootnote{}
  \footnote{#1}
  \addtocounter{footnote}{-1}
  \end{NoHyper}
}

\newcommand*\diff{\mathop{}\!\mathrm{d}}

\title{Back to the Baseline: \texorpdfstring{\\}{ }
       Examining Baseline Effects on Explainability Metrics}

\author[1,2]{Agustin Picard*}
\author[1,2]{Thibaut Boissin*}
\author[5]{\authorcr Varshini Subhash}
\author[3,4]{Rémi Cadène}
\author[2,3]{Thomas Fel}
\affil[1]{IRT Saint-Exupéry, Toulouse, France}
\affil[2]{Artificial and Natural Intelligence Toulouse Institute, Université de Toulouse, France}
\affil[3]{Carney Institute for Brain Science, Brown University, USA}
\affil[4]{Sorbonne Université, CNRS, France}
\affil[5]{Harvard University}

\date{June 2024}

\begin{document}

\maketitle

\begin{abstract}

Attribution methods are among the most prevalent techniques in Explainable Artificial Intelligence (XAI) and are usually evaluated and compared using Fidelity metrics, with Insertion and Deletion being the most popular.
These metrics rely on a baseline function to alter the pixels of the input image that the attribution map deems most important.
In this work, we highlight a critical problem with these metrics: the choice of a given baseline will inevitably favour certain attribution methods over others. More concerningly, even a simple linear model with commonly used baselines contradicts itself by designating different optimal methods.
A question then arises: \textit{which baseline should we use?}
We propose to study this problem through two desirable properties of a baseline: \textit{(i)} that it removes information and \textit{(ii)} that it does not produce overly out-of-distribution (OOD) images.

We first show that none of the tested baselines satisfy both criteria, and there appears to be a trade-off among current baselines: either they remove information or they produce a sequence of OOD images.
Finally, we introduce a novel baseline by leveraging recent work in feature visualisation to artificially produce a model-dependent baseline that removes information without being overly OOD, thus improving on the trade-off when compared to other existing baselines. Our code is available at \href{https://github.com/deel-ai-papers/Back-to-the-Baseline}{\nolinkurl{https://github.com/deel-ai-papers/Back-to-the-Baseline}} 
\end{abstract}

\section{Introduction}

\nnfootnote{\hspace*{-2mm}* Equal contribution}

\begin{figure}[ht]
\centering
\centerline{\includegraphics[width=1.05\textwidth]{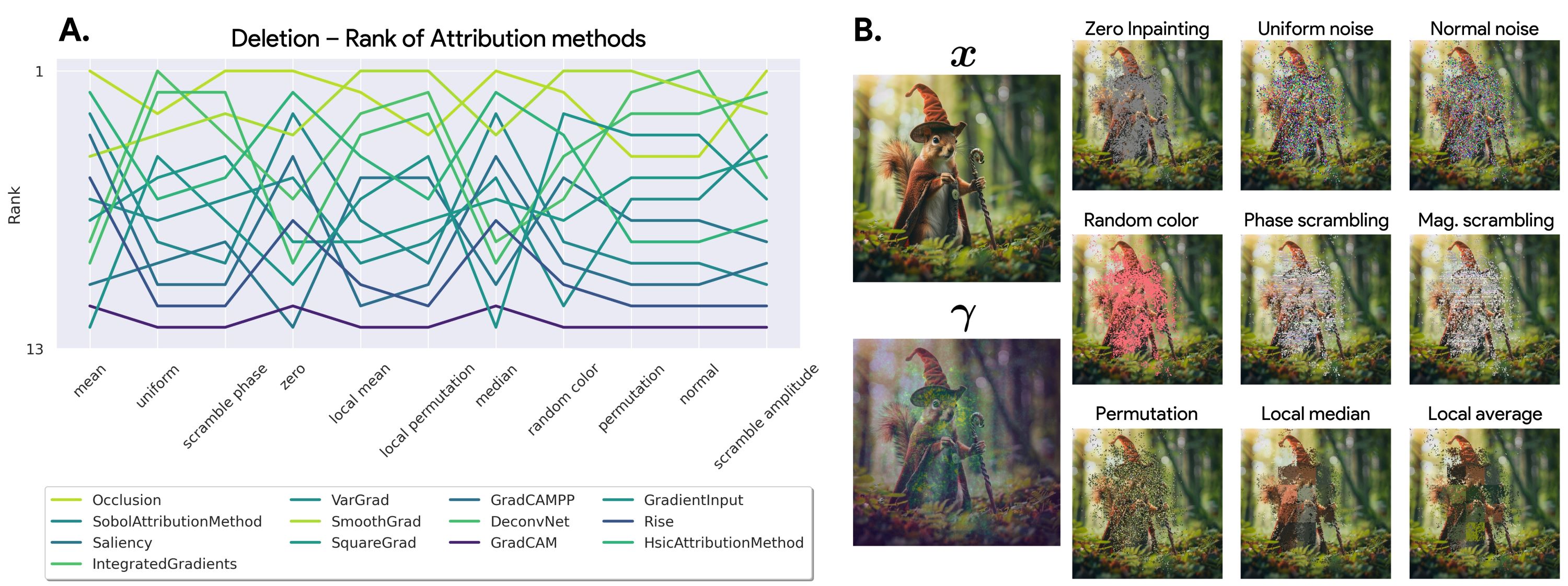}}
\caption{\textbf{A) Deletion ranking is highly sensitive to the baseline.} We study the impact of the choice of the \textit{baseline} on measuring the \textit{faithfulness} of attribution methods for the same ResNet50~\cite{he2016deep} trained on ImageNet1k~\cite{krizhevsky2012imagenet}(See Appendix~\ref{apx:vit-results} for results on a ViT~\cite{dosovitskiy2020image}). Our analysis reveals that the performance of attribution methods under the Deletion metric~\cite{petsiuk2018rise} is highly dependent on the chosen baseline, suggesting that the selection of a baseline can be manipulated to favor particular methods. \textbf{B) Example of how the baseline modifies an input image.} Given an input-explanation pair $(\x, \bm{\gamma})$, we replace the most important pixels of $\bm{x}$ given by $\bm{\gamma}$ with a series of different baselines.}
\label{fig:bigpicture}
\end{figure}

Recent advances in the field of Machine Learning have led to models capable of solving increasingly complex tasks, albeit at the expense of opacity in the strategies they implement~\cite{doshivelez2017rigorous,jacovi2021formalizing,kaminski2021right,kop2021eu}.
Out of this need, a new field was created -- eXplainable Artificial Intelligence (XAI) --, and within, plenty of different approaches to understanding these so-called ``black-boxes'' cropped up~\cite{chen2019looks,kim2018interpretability,ghorbani2019towards,fel2023craft,elhage2022toy,olah2017feature,bricken2023monosemanticity}. One of the most notable is attribution methods~\cite{Zeiler2011,zeiler2014visualizing,GradCAM,novello2022making}, which attempt to explain the model's predictions by finding the pixels or area deemed most important in an image for a given prediction.

The applications of these techniques are diverse -- from aiding in refining or troubleshooting model decisions to fostering confidence in their reliability~\cite{doshivelez2017rigorous}. 
However, a significant drawback of these methods is their susceptibility to confirmation bias: although they seem to provide meaningful explanations to human researchers, they might generate incorrect interpretations~\cite{adebayo2018sanity, ghorbani2017interpretation, slack2021counterfactual}.
In other words, the fact that the explanations are understandable to humans does not guarantee they accurately represent the model's internal processes.
Consequently, the field is actively pursuing improved benchmarks by introducing Fidelity (or Faithfulness) metrics~\cite{samek2015evaluating,petsiuk2018rise, aggregating2020,jacovi2020towards,hedstrom2022quantus,fel2022xplique,kokhlikyan2020captum,hedstrom2023meta}.

These so-called Fidelity metrics measure how closely the explanation reflects the model's local behaviour. This is typically done by using the explanations to carefully pick parts of the input, replacing them with an arbitrary baseline, and then measuring the impact on the model's output. Hence, the outcome of this procedure heavily depends on how the impact on the model is measured precisely, as well as on the choice of the baseline and its interaction with the model and the specificities of the attribution method. We will show later that certain baselines will inherently favour certain types of attribution methods.

While these metrics mark significant progress for the field, one crucial question remains unanswered: each faithfulness metric relies on substituting parts of the input with a \textit{baseline function}. As shown in Figure~\ref{fig:bigpicture}, the choice of this baseline significantly affects the ranking of different attribution methods. This leads to two key questions we aim to address in this work: (i) Can we find the origin of this instability to explain why the ranking using Deletion is unstable while the ranking using Insertion appears stable (see Figure~\ref{fig:insertion-ranking})? (ii) What baseline should we choose?

This paper explores these questions and presents the following findings:

\begin{figure}[t]
\centering
\includegraphics[width=0.75\textwidth]{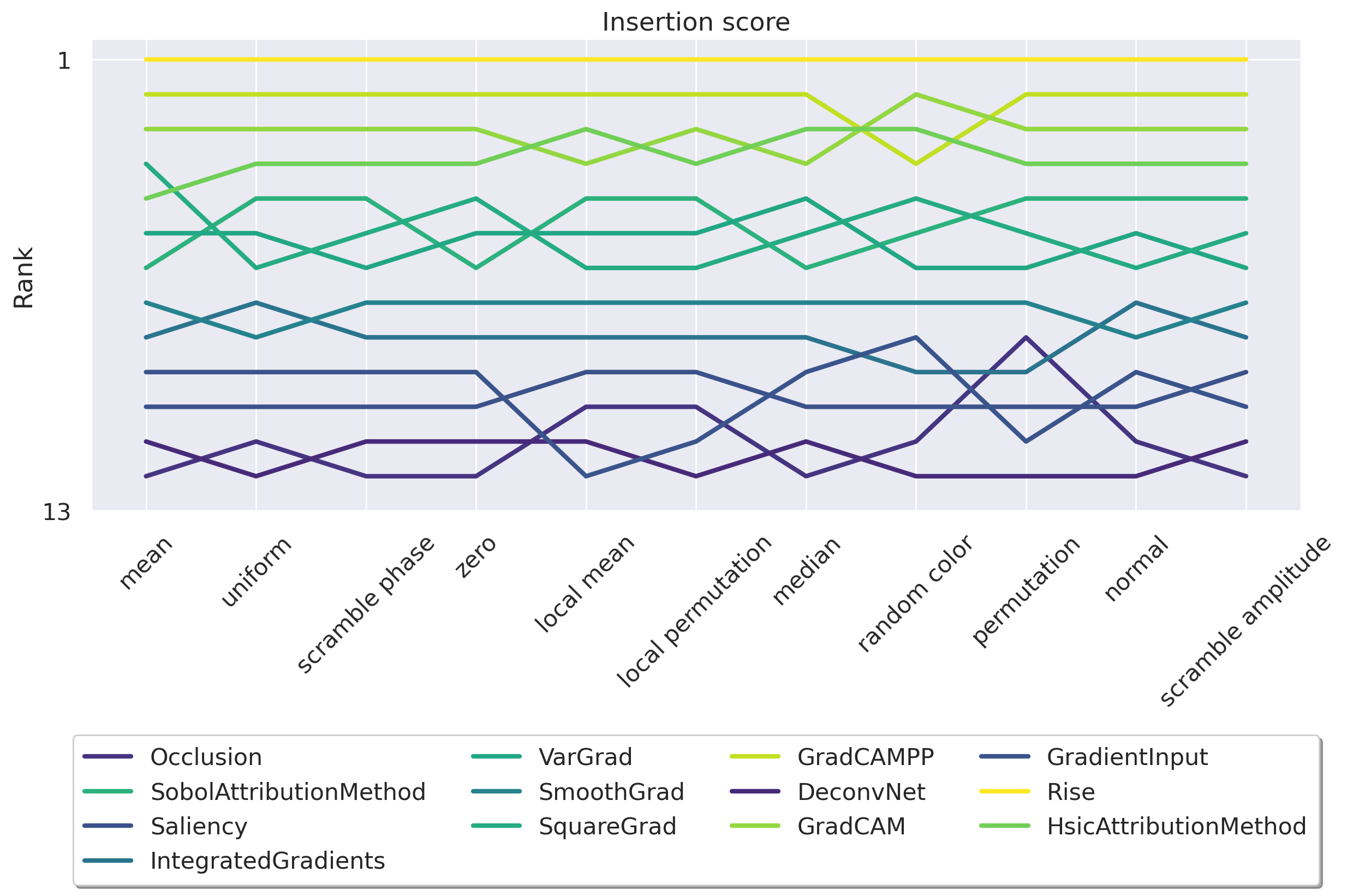}
\caption{
\textbf{Insertion ranking of attribution methods is more stable than Deletion.} We investigate how different baseline choices affect the assessment of attribution method faithfulness for a ResNet50~\cite{he2016deep} model trained on ImageNet1k~\cite{krizhevsky2012imagenet} (See Appendix~\ref{apx:vit-results} for ViT results). We evaluate Insertion across a series of attribution methods from the literature using 11 different baselines (see Appendix~\ref{apx:baselines} for more details) on a total of 14,000 explanations. Although much more stable than Deletion, the OOD-ness of the different baselines makes this metric unreliable.
\vspace{-3mm}
}
\label{fig:insertion-ranking}
\end{figure}

\begin{itemize}[leftmargin=10mm,itemsep=1mm]
\item We begin by demonstrating that the Deletion metric is highly dependent on the baseline choice, rendering it unreliable for comparing attribution methods. In contrast, the Insertion metric appears stable, but it hides a critical flaw: 80\% of its score is achieved in the high out-of-distribution (OOD) regime. Conversely, 80\% of Deletion's score occurs in a non-OOD regime, raising concerns about the reliability of Insertion.

\item We confirm this instability theoretically: even with a simple theoretical example, we demonstrate that for a linear class of models, the optimal attribution for Deletion and Insertion can vary significantly. Depending on the chosen baseline, the optimal attribution could be Saliency~\cite{saliency}, Occlusion~\cite{zeiler2014visualizing}, RISE~\cite{petsiuk2018rise}, Gradient-Input~\cite{shrikumar2017learning}, or Integrated Gradients~\cite{IntegratedGradients}.

\item After investigating the impact and the source of instability, we propose to evaluate the various baselines through two desiderata identified earlier: the out-of-distribution (OOD) score of a baseline and its information removal score. Essentially, the ability to remove information without causing a significant distributional shift. By leveraging quantitative measures, we observe an inherent trade-off: current baselines cannot efficiently remove information without shifting the data distribution to an unacceptable degree, rendering it out-of-distribution (OOD).

\item We conclude that an effective baseline must be model-dependent to adhere to the OOD (Out-Of-Distribution) property and image-independent to decouple it from potential biases, thereby fulfilling the information removal requirement. To meet these criteria, we propose an optimization-based baseline leveraging recent advancements in Feature Visualization~\cite{olah2017feature}. This approach reframes the original feature visualization objective to generate images that activate zero features in the penultimate layer of the model, thereby ensuring that information is removed while keeping the data within the distribution.
\end{itemize}

\section{Related works}

\paragraph{Attribution methods.}

These methods explain a model's explanations by highlighting the parts of the input that are the most important for the predictions, with two main families of methods: black-box or perturbation-based techniques~\cite{ribeiro2016i,lundberg2017unified,zeiler2014visualizing,fel2021sobol,novello2022making,petsiuk2018rise}, and white-box or gradient-based methods~\cite{saliency,IntegratedGradients,smilkov2017smoothgrad,springenberg2014striving,CAM,GradCAM,zhang2024opti}. During its initial stages, attribution methods were mostly tested qualitatively, with a non-negligible part of confirmation bias. This is why these white-box attribution methods have found themselves under scrutiny in various ``sanity checks''. The original sanity check~\cite{adebayo2018sanity,tomsett2019sanity} studied the behavior of these attribution methods when the model's weights were progressively randomized and when the model was trained with random labels (thus, no actual information about the classes could be learned), hence finding the techniques that didn't elucidate the model's inner workings and only depending on the input image. Several other works extend this study, adding their own new desiderata~\cite{hedstrom2024sanity,binder2023shortcomings,hedstrom2023meta}.

\paragraph{Metrics in XAI.}

Although a plethora of metrics exist for attribution methods~\cite{chen2022what,alvarezmelis2018robust,agarwal2022rethinking,fel2021cannot,fel2020representativity,jacovi2020towards,aggregating2020,samek2015evaluating,jacovi2023diagnosing}, in this work we will focus on faithfulness.
This property describes how well an explanation reflects the actual decision-making processes of the model. This is crucial, as explanations need to accurately convey the model's reasoning to users for them to understand the model. In particular, it is often estimated via the Deletion and Insertion metrics~\cite{samek2016evaluating,petsiuk2018rise} Essentially, they propose that an explanation is more faithful to the model if it can accurately identify the pixels that make the prediction's logit decrease the fastest when replaced with a \textit{baseline value}, or increase the quickest when filling an empty image -- i.e. with only a \textit{baseline value} -- with the most important pixels, respectively. Other similar approaches exist for computing fidelity~\cite{rieger2020irof,ancona2017better,yeh2019infidelity}, but we will concentrate on Insertion and Deletion in this work. The community hasn't just stopped at introducing metrics but has also proposed whole standalone evaluation frameworks~\cite{kim2021hive,hedstrom2023meta,agarwal2022openxai}.

\paragraph{The OOD problem in XAI and its metrics.}

The presence of OOD data should be taken into account when producing explanations that require a baseline, as choosing one will have a considerable impact on the outcome, as demonstrated in~\cite{kindermans2019reliability} for various gradient-based attribution methods. The question of the impact of this choice is also discussed in~\cite{sturmfels2020visualizing} for the case of Integrated Gradients~\cite{IntegratedGradients}, and it specifically affects black-box methods that require a perturbation strategy (with a baseline at its core)~\cite{Fong_2017}.
Hooker et al.~\cite{hooker2018benchmark} argue that using OOD data for creating or evaluating explanations makes it difficult to determine if performance degradation is due to distribution shift or actual information removal. In~\cite{hase2021out}, they go even further and explain that this creates what Jacovi et al.~\cite{jacovi2023diagnosing} call \textit{social misalignment}
They argue that an OOD sample depends on model prior and random initialization, contrary to user expectations, and this also affects metrics evaluating explanation quality. Calling this \textit{missingness bias}, Jain et al.~\cite{jain2022missingness} study how the problem of efficient information removal is slightly better posed in the case of ViTs~\cite{dosovitskiy2020image} where the notion of patches/tokens is well-defined for the model.

More recently, and in line with our work, Blucher et al.~\cite{blucher2024decoupling} investigated the impact of the baseline on black-box attribution methods and proposed a method for consistent faithfulness evaluation.  
When computing the Deletion and Insertion metrics, they found that the ranking of the six attribution methods varied with different baselines. In order to mitigate this, they introduced the SRG metric. However, when testing on more (and different) attribution methods and with a broader choice of baseline, we observe that this is not the case entirely (see Appendix~\ref{apx:unstable-srg}), indicating that there are other phenomena taking place under the hood. In this paper, we study this behaviour and uncover some of the reasons why these metrics are unreliable in their current state.

\section{Why Faithfulness rankings are baseline-dependent}\label{sec:insertion-deletion}

The objective of this section is to demonstrate, both qualitatively and theoretically, that the ranking - i.e., the performance of attribution methods with respect to the deletion and insertion metrics - is highly sensitive to the choice of the baseline function.

\paragraph*{Notations}
Throughout, we consider a general supervised learning setting, with an input space $\X \subseteq \R^d$, an output space $\Y \subseteq \R$, and a predictor $\f \in \mathfrak{F}$, which maps any input $\x \in \X$ to an output $\f(\x) \in \Y$. An attribution method, denoted by $\explainer$, is defined as a functional $\explainer : \X \times \mathfrak{F} \to \R^d$ that takes an input and a predictor, returning a score for each input dimension $\explainer(\x, \f)$. A higher score indicates greater importance for a feature.
We denote by $ \u \in \U \subseteq \{1, \ldots, d\}$ a sequence of indices to be removed and put into a baseline state, with $ \u_{1:i} $ representing the first $i$ elements of the sequence and $ \neg \u $ the complement of $ \u $ on $ \U $.
Finally, a baseline function $\baseline : (\X \times \U) \to \X$ is a function that takes a specific input $\x$ and a subset of indices $\u$, replacing these indices according to a specific scheme, such as inpainting or adding noise. The result, $\baseline(\x, \u)$, is the input $\x$ with all indices in $\u$ set to a baseline state (see Figure~\ref{fig:bigpicture}, B for a qualitative example of baselines).

\subsection{Sanity Check on Metric's Baseline}

An initial experiment is illustrated in Figure~\ref{fig:bigpicture} A, where, for a ResNet50 trained on ImageNet, we tested over 15 attribution methods and computed the deletion score for each using more than 2000 samples from the ImageNet~\cite{imagenet_cvpr09} validation dataset. A formal definition of the deletion metric is provided in Appendix~\ref{ap:notations_proof}. 

We then varied the baseline during the deletion score calculation to obtain different rankings of the methods based on the baseline used. The various baselines employed are detailed in Appendix~\ref{apx:baselines} and can be categorized into four groups: baselines that add noise (Gaussian or uniform noise), baselines that replace with a single value (median, zero, mean), local baselines that replace parts of the image (local mean, shuffling), and spectral baselines that preserve part of the spectrum or the phase of the image.

The results indicate that, depending on the chosen baseline, the ranking provided by the deletion metric, which is supposed to build confidence in our methods, is highly sensitive, thereby making the metric itself questionable. Similarly, we observe the same instability in recent vision models such as ViT (see Appendix~\ref{apx:vit-results}). Moreover, this instability issue does not seem to be model-specific since, as we will see, even for a simple class of functions (linear model), theoretical disparities in optimal methods arise depending on the chosen baseline.

\subsection{The interplay between baselines and attribution methods}
\label{sec:toy-models}

\begin{figure}[ht]
    \centering
    \centerline{\includegraphics[width=1.03\textwidth]{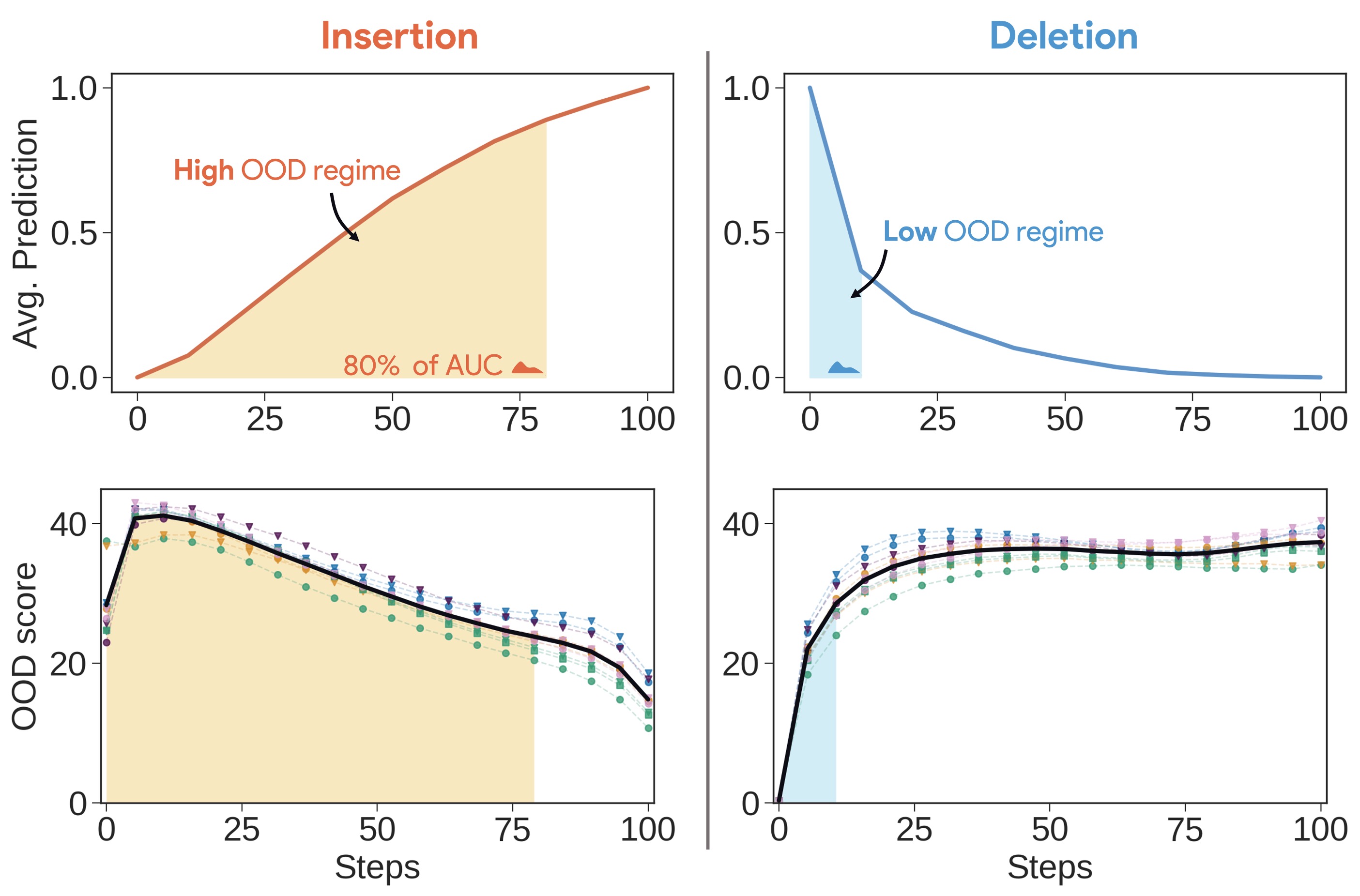}}
    \caption{\textbf{The Cost of stability: the influence of Out-of-Distribution data on Insertion scores.} Insertion and Deletion scores computed on pairs of image-explanation for different baselines. (Top) The highlighted section represents the proportion of deletion and insertion steps (expressed in \%) where 80\% of the cumulative Area Under the Curve (AUC) is achieved. (Bottom) The 1-NN out-of-distribution (OOD) score at each respective step of the Insertion and Deletion metric. We observe that the insertion process predominantly operates on mostly OOD data, in contrast to the deletion process, which primarily concentrates the majority of its score in a smaller number of in-distribution steps.}
    \label{fig:insertion_deletion_tradeoff}
\end{figure}

As seen in the previous section, it is evident that for deep neural networks such as ResNet50 or ViT, the commonly used baselines have a significant impact and can inadvertently favor certain methods. However, it is less clear to what extent this phenomenon holds for simpler models. In fact, we will show that the source of this instability is fundamental, as it occurs even in the case of linear models.
To concretize this, in this section, we will restrict $\f$ to the class of linear functions 

\vspace{-1mm}
$$
\f \in \mathcal{V} \subseteq \mathfrak{F} ~~~\text{with}~~~ \mathcal{V} = \{\x \mapsto \x\w + b \mid \w \in \mathbb{R}^d, b \in \R \}.
$$
\vspace{-1mm}

Under this setup, we can prove that the selection of the baseline changes the optimal method without any particular assumptions on the model's weights. See Appendix~\ref{ap:notations_proof} for all notations.
Indeed, each baseline will impact the trajectory of points used in the deletion and insertion processes, as the most salient pixels are progressively replaced by a baseline. In the case of a linear model, Deletion and Insertion admit a closed form for a fixed baseline, allowing us to evaluate which methods achieve the optimal score.

For example, for the most used baseline, inpainting with zero value as a baseline, we can prove the optimality of some methods:
\begin{proposition}[Gradient-Input, Rise, Occlusion and Integrated Gradients attributions are optimal when $\baseline$ is zero.] \label{prop:optimal_inpainting}
    When $\baseline(\x, \u)$ is the zero function that applies zero value to remove features:
    \vspace{1mm}
    $$
    \baseline(\x, \u)_i = 
    \begin{cases} 
    0  & \text{if } i \in \u, \\
    x_i & \text{otherwise}.
    \end{cases}
    $$
    Then, for the linear model defined in \ref{ap:notations_proof}, the set of attributions that minimizes the Deletion score is Gradient-Input, Integrated Gradient, Rise, and Occlusion.
\end{proposition}

However, the same proof strategy with a uniform noise baseline yields different optimal methods: 

\begin{proposition}[Saliency and Smoothgrad attributions are optimal when $\baseline$ is Uniform noise.] \label{prop:optimal_uniform}When $\baseline(\x, \u)$ add uniform noise to perturb  a features:
\vspace{1.5mm}
$$
\baseline(\x, \u)_i = 
\begin{cases} 
x_i + \xi & \text{if } i \in \u, \\
x_i & \text{otherwise}.
\end{cases}
$$
With $\xi \sim \text{Uniform}([0, 1])$.
Then, the set of attributions that minimize the Deletion score is Saliency and SmoothGrad.
\end{proposition}

We refer the reader to Appendix~\ref{apx:toy-proofs} for the proofs' details. A corollary also shows that the optimal methods are the same in the case of insertion. Although simple, this case illustrates that this instability is at the core of our metrics.

\subsection{Trade-offs for Achieving Metric Stability}

\begin{figure}[ht]
    \centering
    \centerline{\includegraphics[width=1.1\textwidth]{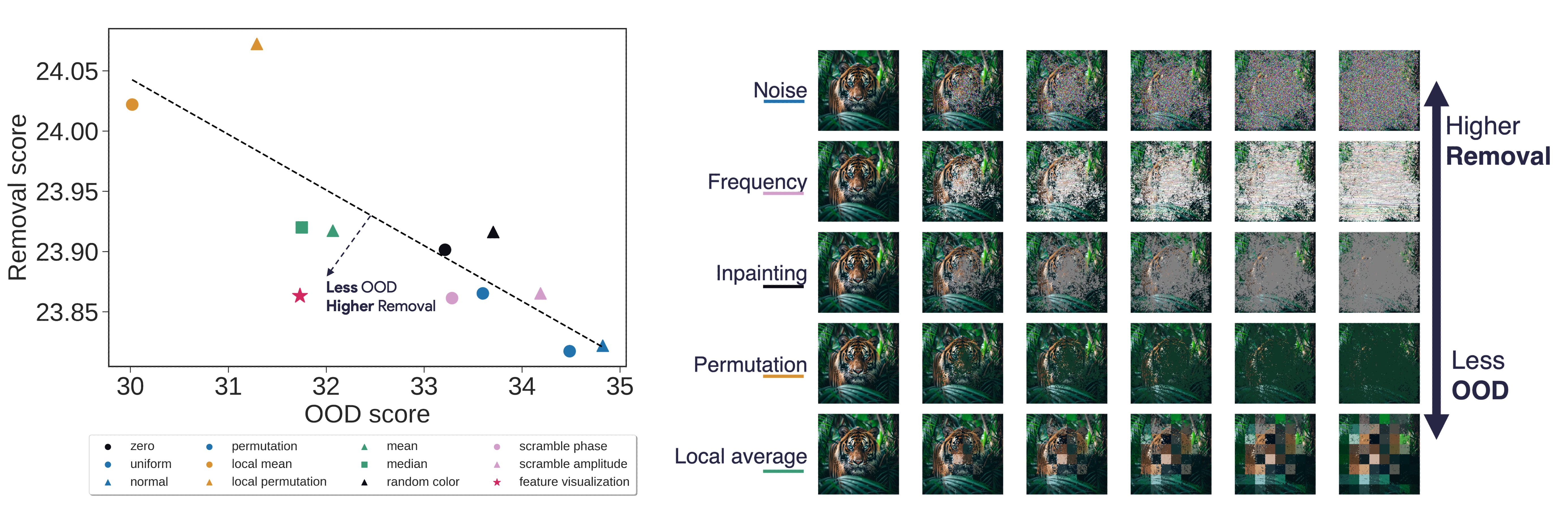}}
    \caption{\textbf{Efficient information removal pushes data Out-of-Distribution (OOD) faster:} Different baselines vary in their ability to remove information (measured by mean logit values across the dataset, where lower values denote greater effectiveness). At the same time, these baselines influence the data's OOD characteristics (evaluated using the 1-NN OOD score). Baselines that reduce the presence of OOD are less effective in removing information, indicating a trade-off between efficient information removal and the induction of OOD. This trade-off is evident both quantitatively (left) and qualitatively (right).}
    \label{fig:qualitative-paths}
\end{figure}

\paragraph{Deletion is sensitive to the baseline.}
In more complex scenarios, particularly when working with deep neural networks, the coupling between attribution methods and baselines is much harder to study, as the way the baseline operation propagates through the network can be difficult to predict. However, this phenomenon can be empirically studied through the instabilities it creates in the Deletion metric (see Figure~\ref{fig:insertion-ranking} (top)).
From Propositions~\ref{prop:optimal_inpainting} and \ref{prop:optimal_uniform}, we can see that the closer an attribution method gets to optimally choosing the proper ordering for a sequence of features to be masked, the more aligned the attribution method will be to the metric with a baseline operator $\baseline(\cdot, \cdot)$. In deep neural networks, this will manifest itself by having Deletion scores that vary wildly between one choice of baseline and another, and it's what we observe in Figure~\ref{fig:insertion-ranking} (left). 

\paragraph{The case of Insertion.}
Intuitively, we would expect the same thing to happen with any other metric based on creating a sequence of masks for the model's input -- \textit{e.g.} the Insertion metric. However, we empirically find this not to be the case. Here we encounter yet another interesting phenomenon that can hinder a metric's capacity to measure an attribution method's faithfulness reliably: the initial image in Insertion is Out-of-Distribution (OOD) -- see Figure~\ref{fig:insertion_deletion_tradeoff} --, and it will align the most with attribution methods that will pull the model out of this regime in as few steps as possible. This explains why the attribution methods' Insertion score has a much lower dependence on the baseline, even though it should theoretically be as dependent as Deletion.

\paragraph{Two desiderata for a good baseline.}
This allows us to establish two desiderata for the baseline operator: (i) we wish to remove as much semantic information from the target region of the input -- \textit{i.e.} high removal score -- as possible, and (ii) we want to do so while keeping the input as much In-Distribution (ID) -- \textit{i.e.} low OOD score -- as possible. Grounding these properties on tractable metrics is essential to their enforcement in practice, which is why we tied them to the following quantities:\\
\textbf{Removal of semantic information.} When a neural network perceives informative semantic information in its input, its activations propagate it to the model's output logits, pushing it to predict a class. Based on this, the most straightforward way to measure information removal is to quantify the decrease in the energy of the logits when the input gets progressively replaced by the baseline.\\
\textbf{Deep-kNN.} Arguably, the simplest model-based OOD metric in the literature is Deep-kNN~\cite{sun2022out}. This choice makes perfect sense, as we wish to measure the OOD-ness of images with respect to the specific model under study, rather than against just a distribution in pixel space. In particular, we set $k = 1$ for our experiments, as we want to quantify how much the baseline pushes the sample to be an outlier for the neural network.

\paragraph{Current baselines don't remove information without making OOD images.} Replacing a region of an image with a baseline without taking into account the rest will break its structure and create images that are at least slightly OOD. This issue can be addressed by employing an inpainting model-based baseline; however, as shown in Figure~\ref{fig:husky}, this approach doesn't resolve the problem either. Our objective is to find a reasonable trade-off between information removal and OOD.

\section{Towards a reliable baseline}\label{sec:desiderata}

A reliable estimation of faithfulness requires that a baseline be efficient at removing semantic information from the input image without pushing it into OOD territory. In practice, this requires some additional considerations that should be taken into account when introducing new baselines.\vspace{1mm}
\paragraph{Baselines should not induce bias:}

\begin{figure}[ht]
    \centering
    \centerline{\includegraphics[width=1.05\textwidth]{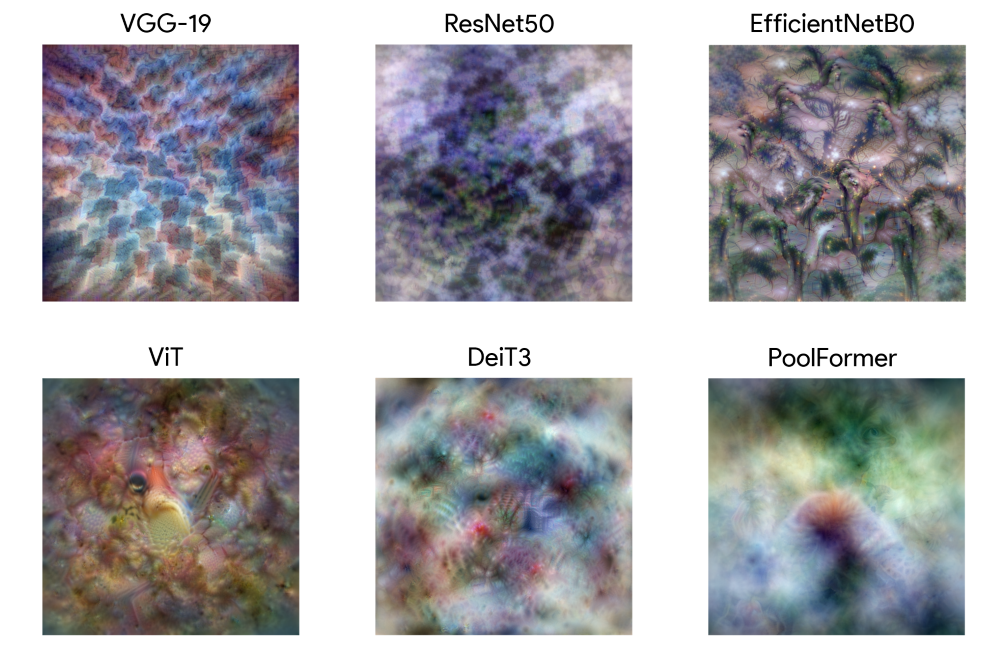}}
    \caption{\textbf{Model-dependent baselines generated using Feature Visualization.} Using MaCo~\cite{fel2023unlocking} to minimize the objective in Equation~\ref{eq:baseline-fv}, we generate images such that the latent space of the model in the penultimate layer is equal to 0. We observe that this image varies vastly depending on the model, effectively showing that the zero of semantic information heavily depends on the model under study. More details can be found in Appendix~\ref{ap:extra_fv}.}
    \label{fig:baseline-fv-small}
\end{figure}

\vspace{1mm}The raison d'être of the baseline in the Insertion and Deletion estimation process is to progressively remove the semantic information that the explanation is highlighting. For this to be the case, a whole image containing only the baseline should lead the model to output logit values that are almost equal for every class -- i.e. there's no actual useful semantic information on which to predict.
We test this property on some of the current baselines to better understand what is happening at the beginning of the Insertion and at the end of the Deletion curves. We observe that the model confidently predicts classes for most of the baselines, meaning that it still sees evidence for some classes in these heavily perturbed images (See Figure~\ref{fig:baseline-classif} in Appendix~\ref{ap:baseline_classif} for additional details).

\paragraph{Why baselines should be In-Distribution (ID) while also being image-independent:}

As discussed previously, one of the desiderata for a baseline for reliable faithfulness metrics is that the resulting images not be Out-of-Distribution from the point of view of the model under study. We have shown that the static baselines in the literature are not really quite ID, with OOD images accounting for 80\% of the Insertion score (see Figure~\ref{fig:insertion_deletion_tradeoff}).

With the rise of increasingly convincing generative models~\cite{ho2020denoising,song2020score,boutin2023diffusion}, it could be enticing to use such models to replace the highlighted parts of the image with non-informative or background segments~\cite{blucher2024decoupling}. This can easily be illustrated using Ribeiro et al.'s famous Huskies vs. Wolves problem~\cite{ribeiro2016i}: suppose our model predicts based only on the presence of snow on the background, then the generative model will replace the snowy background with snowy background -- i.e. leave the important semantic information in the image -- and the canine with snowy background -- i.e. it will add even more important semantic information. It is clear that this means that this sort of baseline only seems appropriate when the model does not rely on background information to make a decision -- i.e. actually focuses on the canine -- but this should not be a requirement for an objective \textbf{explainability} metric.

In addition, if the baseline is conditionally generated on a given image, the perturbed image will still contain semantic information about it, leading to sub-optimal information removal. An example of this can be found in Appendix~\ref{ap:extran_diffusion_example}

\subsection{Why baselines should be model-dependent}

\begin{wrapfigure}{r}{0.5\textwidth}
    \centering
    \includegraphics[width=0.48\textwidth]{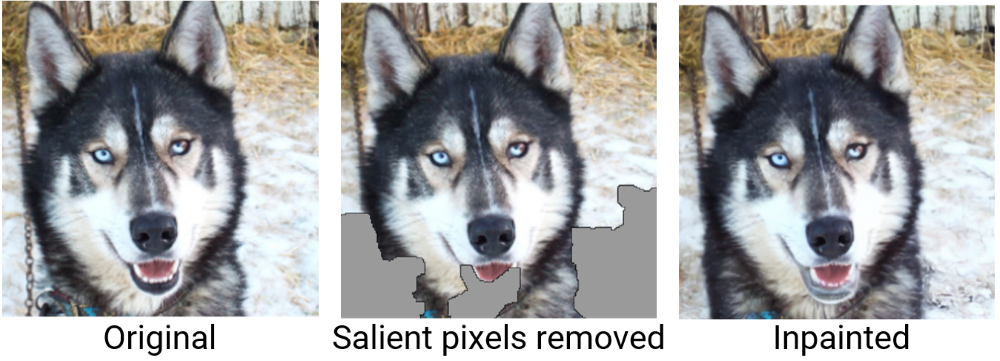}
    \caption{\textbf{Baselines should be image-independent.} For an image of a Husky \cite{ribeiro2016i}(left) that the model classifies as such based on the presence of snow in the background, during Deletion, the baseline will target the background first (middle) and replace it with a reconstruction of the snowy background (right).}
    \label{fig:husky}
\end{wrapfigure}

In the literature, when different baselines were introduced, the idea they had in mind was to replace parts of a natural image with placeholders that appear empty to humans. It is quite natural for people to instinctively think that black or gray non-descript parts of an image are hiding information. However, neural networks have priors and have been trained on specific distributions, and as such, these "patches" actually do contain information (see Figure~\ref{fig:baseline-classif}). Hence, what would be considered an empty image for them will depend on the specific model under study, its priors, and its training distribution.

In particular, creating an image such that the target model outputs null activations at its head would fill the role of carrying zero semantic information for the model. As for the second property, if its activations were actually zero, then its deep 1NN score would be essentially bounded by the training sample with the smallest norm. Fig~\ref{fig:qualitative-paths} show that this corresponds to an interesting trade-off.

\vspace{-1mm}
\subsection{A Feature Visualization perspective}
\label{sec:baseline-fv}

Having established our two key objectives -- i.e., efficient information removal and a low OOD score -- and our three additional requirements, we can create a new baseline that adheres to these precepts. In particular, we will formulate the baseline as an optimization problem to solve using only our model, which will allow us to achieve a better trade-off than the rest of the baselines.
\vspace{-2mm}
\paragraph{The objective introduced in the previous section can be formulated as feature visualization.}
Originally introduced as a means to understand the feature extraction pipeline implemented by convolutional neural networks~\cite{saliency,nguyen2016synthesizing,nguyen2017plug}, these methods propose to optimize an image such that it maximizes a given criterion $\mathcal{L} : \X \to \R$ -- usually using some (intermediary) values inside the neural network (neurons, layers, logits, etc.). In essence, the objective is to find an image $\x^\star$ that verifies the condition:
\vspace{-1mm}
\begin{equation}
    \x^\star = \argmax_{\x \in \X} \mathcal{L}(\x) - \lambda \mathcal{R} (\x) \, ,
\end{equation}
where $\mathcal{R} : \mathcal{X} \to \mathbb{R}$ is a regularization term to guide the solutions towards more natural-looking images, and $\lambda$ is a hyperparameter to balance the main criterion and the regularizer.

Several strategies have been proposed in the literature to try to render feature visualizations that follow a structure that resembles that of natural images, be it through data augmentation during the optimization process, through filtering techniques~\cite{tyka2016class} or by leveraging properties in the Fourier domain~\cite{olah2017feature}. The latest addition is probably MaCo~\cite{fel2023unlocking}, which improves the realism of model-free feature visualizations and allows them to scale up to the largest vision models, and does so by constraining the frequency magnitude to be equal to that of the mean of the dataset and only optimize the phase of the image. In practice, we will use this technique to generate our baselines using the formulation below. \\
Instead of generating an image that maximizes the activations of certain parts of the model, we will use feature visualization to perform a task similar to feature inversion (as described in~\cite{fel2023unlocking}): we will choose the criterion $\mathcal{L}_v$ so that the model's last layer's input is zero -- \textit{i.e.} the information being fed into the last layer of the model is zero and only the bias of this linear layer produces the non-zero logits for the model. In addition, we will perform the optimization operation in the Fourier domain as it provides an efficient pre-conditioning for generating more natural images. Combining all this, we obtain the following objective:

\vspace{-1.5mm}
\begin{equation}
    \bm{\varphi}^{\star} = \argmin_{\bm{\varphi} \in \R^d} \| (\f_\ell \circ \mathcal{F}^{-1})(\bm{r} e^{i \bm{\varphi}})  \|_2^2 , ~~~ \text{and} ~~~ \x^{\star} = \mathcal{F}^{-1}(\bm{r} e^{i \bm{\varphi}^{\star}})
    \label{eq:baseline-fv}
\end{equation}
where $\f_\ell$ represents the model up to the penultime layer, and $\mathcal{F}^{-1}$ is the inverse Fourier transform. This corresponds to creating an image containing zero information from the model's point of view, which is also the definition of the baseline value in faithfulness metrics. \\
We can expect this generated image to not induce bias, as, by construction, it will bring the model's activations to zero and to be ID to the model, as it was generated using the information in the networks' weights. It is interesting to note that every model will have learned a specific set of features during training, and finding a baseline that corresponds to zero information for every model is impossible.

\section{Experimental setup}\label{sec:results}

To gather all the different insights, we performed extensive tests on more than 20,000 explanations generated for Computer Vision models (ResNet50V2~\cite{he2016deep}, ViT~\cite{dosovitskiy2020image}) and benchmarked using Deletion and Insertion~\cite{petsiuk2018rise} with 11 different baselines~\ref{apx:baselines}. In particular, we monitored the scores of OOD and information removal during the computation of these two metrics and at the end of the process -- i.e. when the input image has been completely replaced by the baseline. For these experiments, the \textbf{Xplique} library~\cite{fel2022xplique} was used, as it provides the necessary blocks, and we generated the explanations using the default settings for all attribution methods. Additional results on a Vision Transformer~\cite{dosovitskiy2020image} are presented in Appendix~\ref{apx:vit-results}.\\
By computing the OOD score of the image as it gets progressively replaced by the baseline, we can determine the percentage of the Insertion and Deletion metrics that are calculated on high OOD data and, thus, unreliable. In Figure~\ref{fig:insertion_deletion_tradeoff}, we observe that on average, 80\% of the Insertion score is actually estimated on data that is Out-of-Distribution, compared to Deletion, which exhibits the opposite behavior (most of the AUC is computed on mostly In-Distribution data). This enables us to conclude that Insertion is mainly governed by the model's manner of dealing with OOD images, while Deletion shows it could be reliable, provided the baseline is able to remove information with a high degree of precision. \\
We also calculate the information removal of each baseline, which allows us to compare the different baselines on two axes. In Figure~\ref{fig:qualitative-paths} (right), we observe how each baseline progressively removes information from an image and how the target gets increasingly difficult to recognize, with some baselines visually destroying more information than others. Further, it is possible to analyze this quantitatively, and we observe in Figure~\ref{fig:qualitative-paths} (left) that, for the static baselines in the literature, the more efficient the information removal, the more OOD the inputs become for the model. In contrast, we observe that our model-dependent baseline provides a better trade-off between these scores: it suppresses information more efficiently while staying closer to the training distribution. \\
The findings from our tests highlight the pivotal role that the choice of baselines plays in evaluating the faithfulness of attribution methods. By employing baselines that offer a more balanced trade-off between information removal and plausibility, the research community can achieve a more accurate and reliable measurement of model faithfulness. This advancement in baseline selection steers us toward developing more robust and transparent XAI systems. Ultimately, this progress promises to refine our methodologies and improve the credibility of the insights derived from XAI evaluations.

\section{Discussions, Limitations and Future work}

In this work, we investigated some limitations on XAI faithfulness metrics by highlighting some glaring issues in their computation. Namely, Insertion and Deletion, two of the most popular metrics in the field, call for a baseline value, which we show heavily impacts the results and rankings of the attribution methods we tested. We chose to study two aspects to better characterize how such a baseline affects the results: information removal and OOD score. We introduce a principled baseline that is model-specific and show that it provides a better trade-off on the aforementioned aspects. Besides metrics, some black-box attribution methods~\cite{novello2022making,fel2021sobol,zeiler2014visualizing,ribeiro2016i} also rely on a baseline to mask parts of the input and, hence, can benefit from the insights of this paper.

By searching for maximal information removal, we have concluded that baselines need to be image-independent. However, this will inexorably lead to the trade-off we observe in Figure~\ref{fig:qualitative-paths}. In their current state, dynamic baselines are not trustworthy enough to be employed in the evaluation of attribution methods, but more advanced, reliable conditional baselines could be devised in the future.

Recent work~\cite{geirhos2023don} has demonstrated that due to the non-convex nature of feature visualization optimization, explanations from this technique lack complete reliability since the optimizer cannot be guaranteed to find the optimal solution. However, this limitation has not prevented our model-dependent baseline from offering a better trade-off, and its performance should evolve with future advances in the field.

Finally, we invite readers to explore the appendices for a comprehensive set of additional results. In particular, we present the experimental setup in detail~\ref{sec:results}, some additional examples of Feature Visualization baselines~\ref{ap:extra_fv}, additional information, closed forms, and the proofs of our theoretical results~\ref{apx:toy-proofs}, our experiments on the SRG metric~\ref{apx:unstable-srg}, our experiments on Vision Transformers~\ref{apx:vit-results}, a description of the baselines~\ref{apx:baselines}, some additional curves on the evolution of the OOD score for Deletion~\ref{apx:ood-curves}, a simple experiment showing that the baselines still contain information from the model's point of view~\ref{ap:baseline_classif}, and an additional example of the issues of inpainting diffusion models as baselines~\ref{ap:extran_diffusion_example}.

\section{Acknowledgements}
This work was carried out within the DEEL project,\footnote{\url{https://www.deel.ai/}} which is part of IRT Saint Exupéry and the ANITI AI cluster. The authors acknowledge the financial support from DEEL's Industrial and Academic Members and the France 2030 program – Grant agreements n°ANR-10-AIRT-01 and n°ANR-23-IACL-0002. Additional support provided by ONR grant N00014-19-1-2029 and NSF grant IIS-1912280. Support for computing hardware provided by Google via the TensorFlow Research Cloud (TFRC) program and by the Center for Computation and Visualization (CCV) at Brown University (NIH grant S10OD025181).

\bibliography{main}
\bibliographystyle{plain}


\appendix


\newpage

\section{Examples of Feature Visualization Baselines}\label{ap:extra_fv}

Basing ourselves on the MaCo~\cite{fel2023unlocking} implementation in Xplique~\cite{fel2022xplique}, we adjust the optimization objective as follows:

\begin{equation}
    \bm{\varphi}^\star = \argmin_{\bm{\varphi} \in \R^d} \| (\f_\ell \circ \mathcal{F}^{-1})(\bm{r} e^{i \bm{\varphi}})  \|_2^2 , ~~~ \text{and} ~~~ \x^\star = \mathcal{F}^{-1}(\bm{r} e^{i \bm{\varphi}^{\star}})
\end{equation}

for $\f_\ell$ the model's activations at layer $\ell$. We generate our baselines with $\ell$ the penultimate layer -- i.e. just before the last layer.

We create baselines for some models, and we showcase them in Figure~\ref{fig:baselines_fv_apx}. We observe that the architecture, the dataset it was trained on, and the data augmentation schemes have a substantial impact on what the model considers to be an image without any semantic information.

\begin{figure}[H]
\centering
\includegraphics[width=\textwidth]{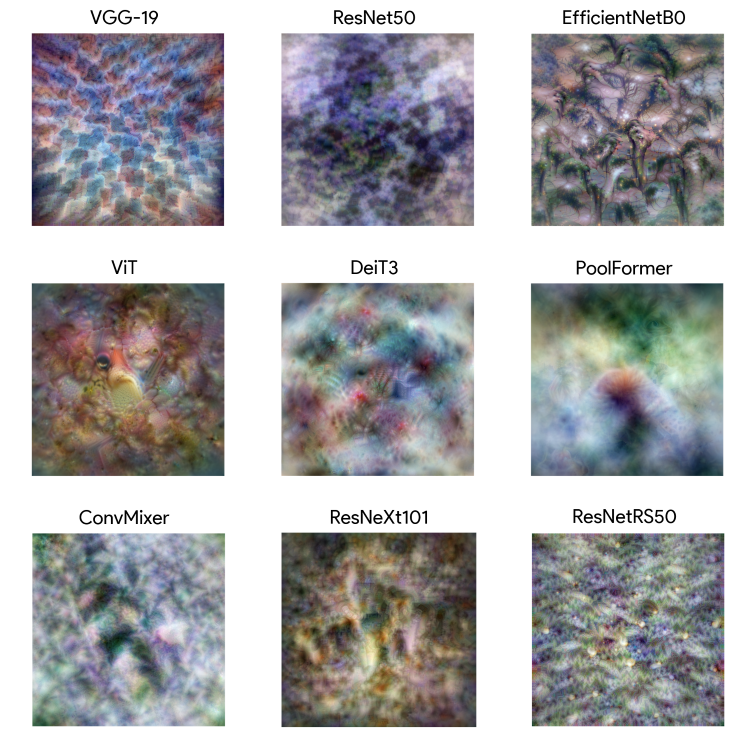}
\caption{\textbf{Model-dependent baselines generated using Feature Visualization.} For nine different models trained or fine-tuned on ImageNet1k~\cite{krizhevsky2012imagenet}, we generate some model-dependent baselines with the procedure described in Section~\ref{sec:baseline-fv}.}
\label{fig:baselines_fv_apx}
\end{figure}

\section{Choosing the best method is baseline-dependent}\label{apx:toy-proofs}

As mentioned previously, we can demonstrate that selecting a baseline is equivalent to choosing the best method for the same toy model. We will consider a linear toy model. We will show that in the case of the zero baseline, any method that gives a ranking of pixels different from the ranking provided by gradient-input~\cite{shrikumar2017learning} will have a lower score.

\paragraph{Notations:} \label{ap:notations_proof} Let the toy model $\f \in \mathcal{V} \subseteq \mathfrak{F}$ that belongs to the class of linear functions $\mathcal{V} = \{\x \mapsto \x\w + b \mid \w \in \mathbb{R}^d, b \in \R \}$ with weights $\w \in \R^d$ and bias $b \in \R$. Moreover, we recall that $\u$ denotes the sequence of elements to be removed and placed into a baseline state. $ \u $ then has $ d $ elements, we denote $ \u_{1:i} $ the i first elements of the sequence and $ \neg \u $ the complement of $ \u $ on $ \U $. A baseline function $\baseline : (\X \times \U) \to \X$ is a function that takes a specific input $\x$, a subset of indices to replace $\u \in \U \subseteq \{1, \ldots, d\}$ using a specific scheme, such as zero or adding noise.

We first recall the closed-form solution of various attribution methods in the linear case before formally introducing the Deletion and Insertion metrics.

\paragraph{Closed-form solution for Attribution methods.}

In the linear case, several attribution methods, $\explainer$, have closed forms, which allows us to derive closed forms for the Deletion and Insertion metrics. Specifically, we will focus on six methods: Saliency~\cite{saliency}, Gradient-Input~\cite{ancona2017better}, SmoothGrad~\cite{smilkov2017smoothgrad}, Integrated Gradients~\cite{IntegratedGradients}, Occlusion~\cite{zeiler2014visualizing}, and RISE~\cite{petsiuk2018rise}. For reference, some of these closed forms can be found in \cite{ancona2017better} and \cite{fel2024holistic}.

We start by deriving the closed form of Saliency (Sa):

\begin{flalign*}
\explainer^{(Sa)}(\x, \f) 
&= \nabla_{\x} \f(\x) && \\
&= \nabla_{\x} (\x \w + b) && \\
&= \w.
\end{flalign*}

This simply returns the model's weights. Concerning Gradient-Input (Gi), following the same strategy, we naturally have:

\begin{flalign*}
\explainer^{(Gi)}(\x, \f) 
&= \nabla_{\x} \f(\x) \odot \x && \\
&= \nabla_{\x} (\x \w + b) \odot \x && \\
&= \w \odot \x.
\end{flalign*}

And we can find similar results for Integrated Gradients (Ig):

\begin{flalign*} 
\explainer^{(Ig)}(\x, \f)
&= (\x - \x_0) \odot \int_0^1 \nabla_{\x} \f(\x_0 + \alpha (\x - \x_0)) \diff \alpha && \\
&= \x \odot \int_0^1 \nabla_{\x} (\alpha \x \w + b + (\alpha-1) \x_0 \w) \diff \alpha && \\
&= \x \odot \int_0^1 \alpha \w \diff \alpha && \\
&= \x \odot \w \left[\frac{1}{2} \alpha^2 \right]_0^1 && \\
&= \frac{1}{2} \x \odot \w.
\end{flalign*}

Smoothgrad (Sg) ends up being strictly equal to Saliency:

\begin{flalign*}
\explainer^{(Sg)}(\x, \f)
&= \underset{\bm{\delta} \sim \mathcal{N}(0, \mathbf{I}\sigma)}{\mathbb{E}}(\nabla_{\x} \f(\x + \bm{\delta})) && \\
&= \underset{\bm{\delta} \sim \mathcal{N}(0, \mathbf{I}\sigma)}{\mathbb{E}}(\nabla_{\x}((\x + \bm{\delta}) \w + b)) && \\
&= \underset{\bm{\delta} \sim \mathcal{N}(0, \mathbf{I}\sigma)}{\mathbb{E}}(\w) && \\
&= \w.
\end{flalign*}

And Occlusion (Oc) to Gradient-Input:

\begin{flalign*}
\explainer^{(Oc)}_i(\x, \f)
&= \f(\x) - \f(\x_{[i = x_0]}) && \\
&= (\x \w + b) - (\x_{[i = x_0]} \w + b) && \\
&= (\sum_{j} x_j \w_j) - (\sum_{j \neq i} x_j \w_j) && \\
&= x_i \w_i.
\end{flalign*}

Finally, for RISE, we reach the conclusion that RISE is proportional to Gradient-Input, with a constant added to each of the scores for each component, the constant being $b + \frac{1}{2} (\x \w)$.

\begin{flalign*}
\explainer^{(RI)}_i(\x, \f)
&= \mathbb{E}(\f(\x \odot \bm{m}) | \bm{m}_i = 1) && \\
&= \mathbb{E}((\x \odot \bm{m}) \w + b | \bm{m}_i = 1) && \\
&= b + \sum_{j \neq i} x_j \mathbb{E}(\bm{m}_j) \w_j + x_i \w_i && \\
&= b + \frac{1}{2} (\x \w + x_i \w_i).
\end{flalign*}

\paragraph{Metrics.}

We now recall the original Deletion score, defined as the sum of the individual steps of deletion: 

\begin{definition}[Deletion] Deletion is defined as the Area Under the Curve (AUC) for the model's logit as the most salient pixels of the image get progressively replaced by the baseline. This corresponds to calculating the AUC of the curve 
$$
\Deletion(\x, \u) = \sum_{i=1}^d \deletion_i(\x, \u)_i.
$$
With the $i$-th deletion step defined as the score when only the top $i$ most important features are put to a baseline state:
$$
\deletion(\x, u)_i = \f (\baseline(\x, u_{1:i}))
$$
Where $\u$ is the vector containing indices of $x$ by order of importance as defined by the attribution method under study. Logically, the more faithful the method is, the lower the AUC.
\end{definition}

\begin{definition}[Insertion] Similarly, the Insertion metric is defined as the sum of the $d$ insertion steps:
$$
\Insertion(\x, \u) = \sum_{i=1}^d \insertion(\x, \u)_i.
$$
With the $i$-th insertion step defined as the score when only the top $i$ most important features are \textbf{not} put to a baseline state:
$$
\insertion(\x, \u)_i = \f(\Pi(\x,\neg \u_{1:i})).
$$
\end{definition}

\paragraph{Optimal subset.} We are now interested in finding the optimal subset $\u$ such that, for example, it minimizes the Deletion score:
\begin{flalign}
\nonumber
\argmin_{\u} \Deletion(\x, \u) &= \argmin_{\u} \sum_{i=1}^{d} \deletion(\x,\u)_i  && \\\nonumber
&= \argmin_{\u} \quad 
(\baseline(\x, \u_1)\w  + b) + \ldots + 
(\baseline(\x, \bm{\u})\w  + b)
&& \\ \nonumber
&= \argmin_{\u} \quad 
\baseline(\x, \u_1)\w + \ldots + \baseline(\x, \u)\w
\end{flalign}
Moreover, when we instantiate popular baselines such as inpainting or uniform noise, we can demonstrate that the optimal ordering of the most important variable, the construction of $\u$, is optimal for a different set of attribution methods.

\begin{proposition}[Gradient-Input attributions are optimal when $\baseline$ is zero.] When $\baseline(\x, \u)$ is the zero function that applies zero value to remove a feature:
$$
\baseline(\x, \u)_i = 
\begin{cases} 
0  & \text{if } i \in \u, \\
x_i & \text{otherwise}.
\end{cases}
$$
Then, for the linear model defined in \ref{ap:notations_proof}, the set of attributions that minimizes the Deletion score is Gradient-Input, Integrated Gradient, Rise, and Occlusion.
And more generally, all the attribution methods such that $\forall (i,j) ~~ x_i \cdot w_i \geq x_j \cdot w_j \implies \explanation(\x)_i > \explanation(\x)_j$ are optimal.
\end{proposition}
\begin{proof}

\begin{flalign}
\nonumber
\argmin_{u} \Deletion(\x, \u)
&= \argmin_{\u} \baseline(\x, \u_1)\w + \ldots + \baseline(\x, \u)\w
&& \\ \nonumber
&= \argmin_{\u} d \cdot (x_{\u_1} w_{\u_1}) + (d-1) \cdot (x_{\u_2} w_{\u_2}) + \ldots 1 \cdot (x_{\u_d} w_{\u_d}) && \\ \nonumber
&= \argmin_{\u} \sum_{i=1}^{d} (d-i+1) \cdot (x_{\u_i} w_{\u_i})
\end{flalign}

We will recall the Rearrangement inequality~\cite{day1972rearrangement} which states that for every choice of \textit{real numbers}
\[
x_1 \leq \cdots \leq x_n \quad \text{and} \quad y_1 \leq \cdots \leq y_n
\]
and every \textit{permutation} $\sigma$ of the numbers $1, 2, \ldots, n$ we have
\[
x_1 y_n + \cdots + x_n y_1 \leq x_1 y_{\sigma(1)} + \cdots + x_n y_{\sigma(n)} \leq x_1 y_1 + \cdots + x_n y_n.
\]
We can then conclude that $\Deletion(\x, \u)$ is minimized if and only if $x_i w_i$ are ordered : $ \forall (i,j) : i < j, x_{\u_i} w_{\u_i} \geq x_{\u_j} w_{\u_j}$.
This condition is only satisfied for attribution methods $\explanation$ that are of the form of Gradient-input $\explanation(\x) = \x \cdot \w$, thus, in the case of the linear model: Gradient-Input~\cite{shrikumar2017learning}, Integrated Gradient~\cite{IntegratedGradients}, Occlusion~\cite{zeiler2014visualizing} and RISE~\cite{petsiuk2018rise}.
\end{proof}

\begin{proposition}[Saliency attributions are optimal when $\baseline$ is Uniform noise.] When $\baseline(\x, \u)$ adds uniform noise to perturb a feature:
$$
\baseline(\x, \u)_i = 
\begin{cases} 
x_i + \xi & \text{if } i \in \u, \\
x_i & \text{otherwise}.
\end{cases}
$$
With $\xi \sim \text{Uniform}([0, 1])$.
Then, the set of attributions that minimize the Deletion score is Saliency and SmoothGrad.
And more generally, all the attribution methods such that $\forall (i,j) ~~ w_i \geq w_j \implies \explanation(\x)_i > \explanation(\x)_j$ are optimal.
\end{proposition}
\begin{proof}

\begin{flalign}
\nonumber
\argmin_{\u} \Deletion(\x, \u)
&= \argmin_{\u} \E\big(\baseline(\x, \u_1)\w + \ldots + \baseline(\x, \u)\w)
\big) && \\ \nonumber
&= \argmin_{\u}\E \big( \x\w + \xi w_{\u_1} + \x\w + \xi w_{\u_1} + \xi w_{\u_2} + \ldots + \x\w + \xi \w  \big) && \\ \nonumber
&= \argmin_{\u}\E \big( d \cdot \xi w_{\u_1} + (d-1) \cdot \xi w_{\u_2} + \ldots + 1 \cdot \xi w_{\u_d}   \big) && \\ \nonumber
&= \argmin_{\u}\E \big( \sum_{i=1}^{d} (d-i-1) \cdot (\xi w_{\u_i})
\big) && \\ \nonumber
&= \argmin_{\u} \sum_{i=1}^{d} (d-i-1) \cdot (\E(\xi) w_{\u_i}) && \\ \nonumber
&= \argmin_{\u} \sum_{i=1}^{d} (d-i-1) \cdot (\frac{1}{2} w_{\u_i}) && \\ \nonumber
&= \argmin_{\u} \sum_{i=1}^{d} (d-i-1) \cdot w_{\u_i} && 
\end{flalign}

which, by the Rearrangement inequality~\cite{day1972rearrangement}, is minimized if and only if the $w_i$ are ordered : $ \forall (i,j) : i < j, w_{\u_i} \geq w_{\u_j}$.
This condition is only satisfied for attribution methods $\explanation$ that are of the form of Saliency $\explanation(\x) = \w$, thus, in the case of the linear model: Saliency and SmoothGrad.
\end{proof}

\begin{corollary}[Results holds for Insertion]
 For the linear model defined in \ref{ap:notations_proof}, the optimal methods for the Deletion metric are also optimal for the Insertion metric. This holds for both when $\baseline(\x, \u)$ is the zero baseline (resp. adding uniform noise) function.
\end{corollary}

\begin{proof}
The previous proof strategy can be unrolled, which leads to:
\[
\argmax_{u} \text{Insertion}(x, u) =  \argmax_u \baseline(x, \neg u) + \dots + \argmax_u \baseline(x, \neg u_1)
\]
In this context, we observe that we can apply the other side of the Rearrangement inequality, leading to the same conclusion.
\end{proof}

\section{The instability of SRG}\label{apx:unstable-srg}

Introduced in~\cite{blucher2024decoupling}, Symmetric Relevance Gain (SRG) is a metric that combines the scores of what they call Least Important First (LIF) and Most Important First (MIF) -- another naming convention for what we call Insertion and Deletion. In particular, the SRG was proposed as a solution to Insertion and Deletion's dependence on the baseline (or imputer, in their notation). In their paper, they demonstrate that they do achieve some stability, but their testbed is limited: the baselines on which they tested their techniques find themselves on the left-hand side of the Information removal/OOD score trade-off. In essence, the SRG is computed as follows:

\begin{equation}
SRG (\explainer) = LIF (\explainer) - MIF (\explainer) = Insertion (\explainer) - Deletion (\explainer) \, ,
\end{equation}
for $\explainer$ an attribution method.

\begin{figure}[H]
\centering
\includegraphics[width=\textwidth]{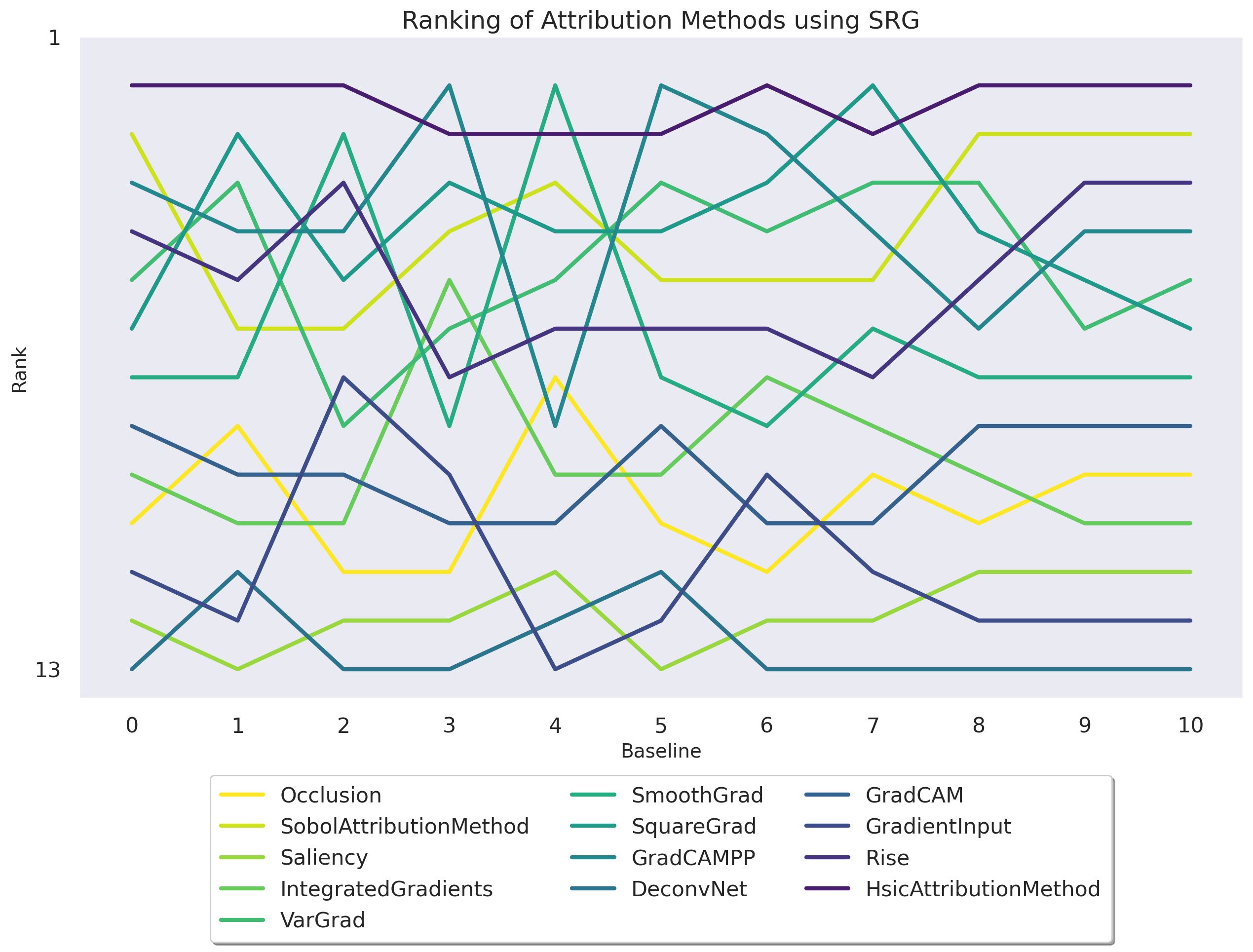}
\caption{\textbf{Ranking of attribution methods using the SRG~\cite{blucher2024decoupling} metric.} Leveraging the Insertion and Deletion metrics computed on 14,000 explanations for a ResNet50V2~\cite{he2016deep} pre-trained on ImageNet1k~\cite{krizhevsky2012imagenet}, we compute the SRG metric for 10 different baselines and 12 different attribution methods. We observe that the metric is still baseline-dependent, unfortunately.}
\label{fig:ranking-srg}
\end{figure}

We observe in Figure~\ref{fig:ranking-srg} that this metric is as unstable as Deletion and Insertion when testing with other static (and equally acceptable) baselines. We can hence attribute this discrepancy with~\cite{blucher2024decoupling} to the lack of high-information-removal baselines in their tests.

\section{Additional results on ViT}\label{apx:vit-results}

In addition to the ResNet50V2~\cite{he2016deep} with which we generated most of the results in the main paper, we have also experimented with a Vision Transformer (ViT)~\cite{dosovitskiy2020image}. In particular, we chose a pre-trained version with a patch size of 16 from the \texttt{vit\_keras} library\footnote{https://github.com/faustomorales/vit-keras}.

\begin{figure}[H]
\centering
\includegraphics[width=0.49\textwidth]{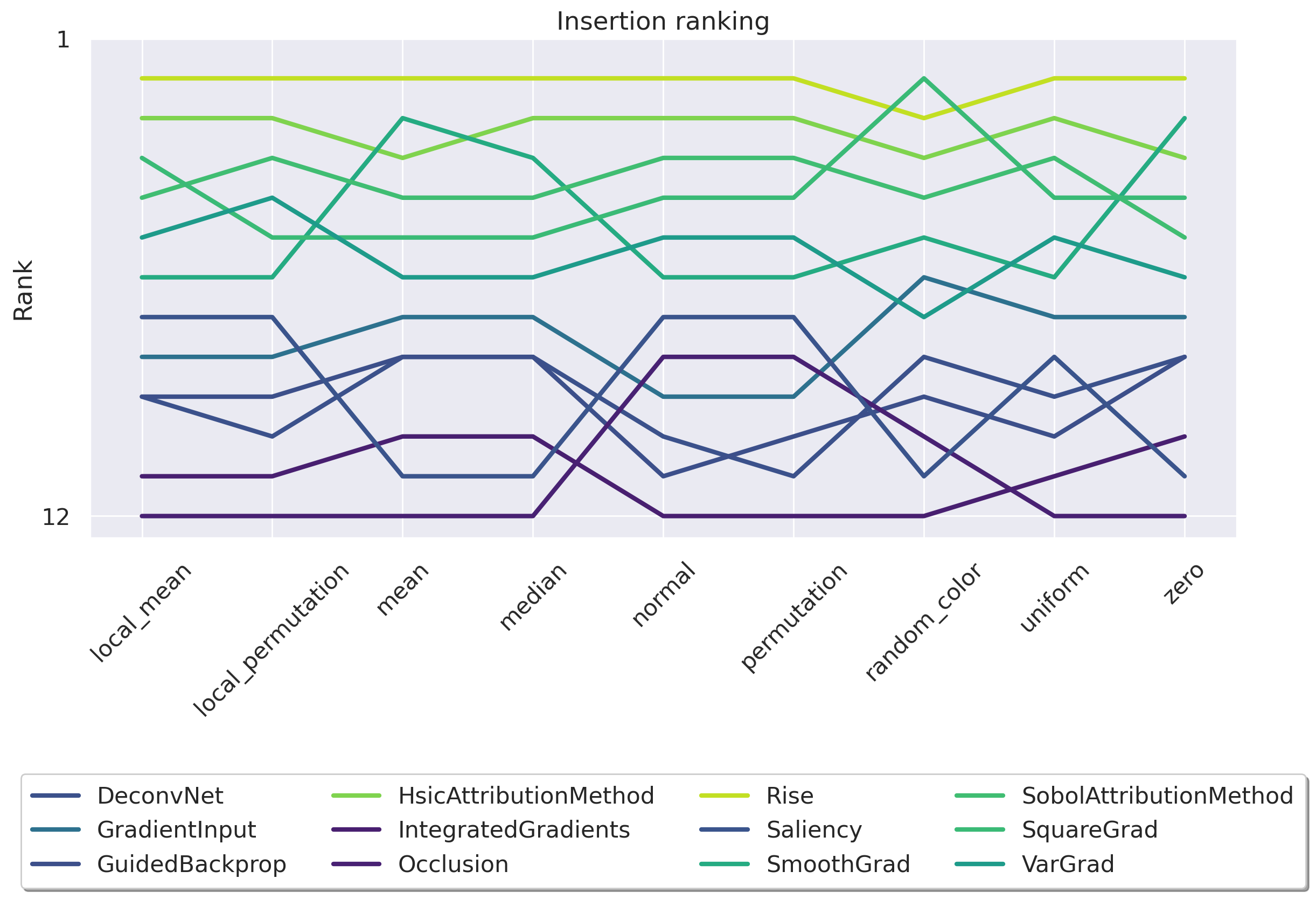}
\includegraphics[width=0.49\textwidth]{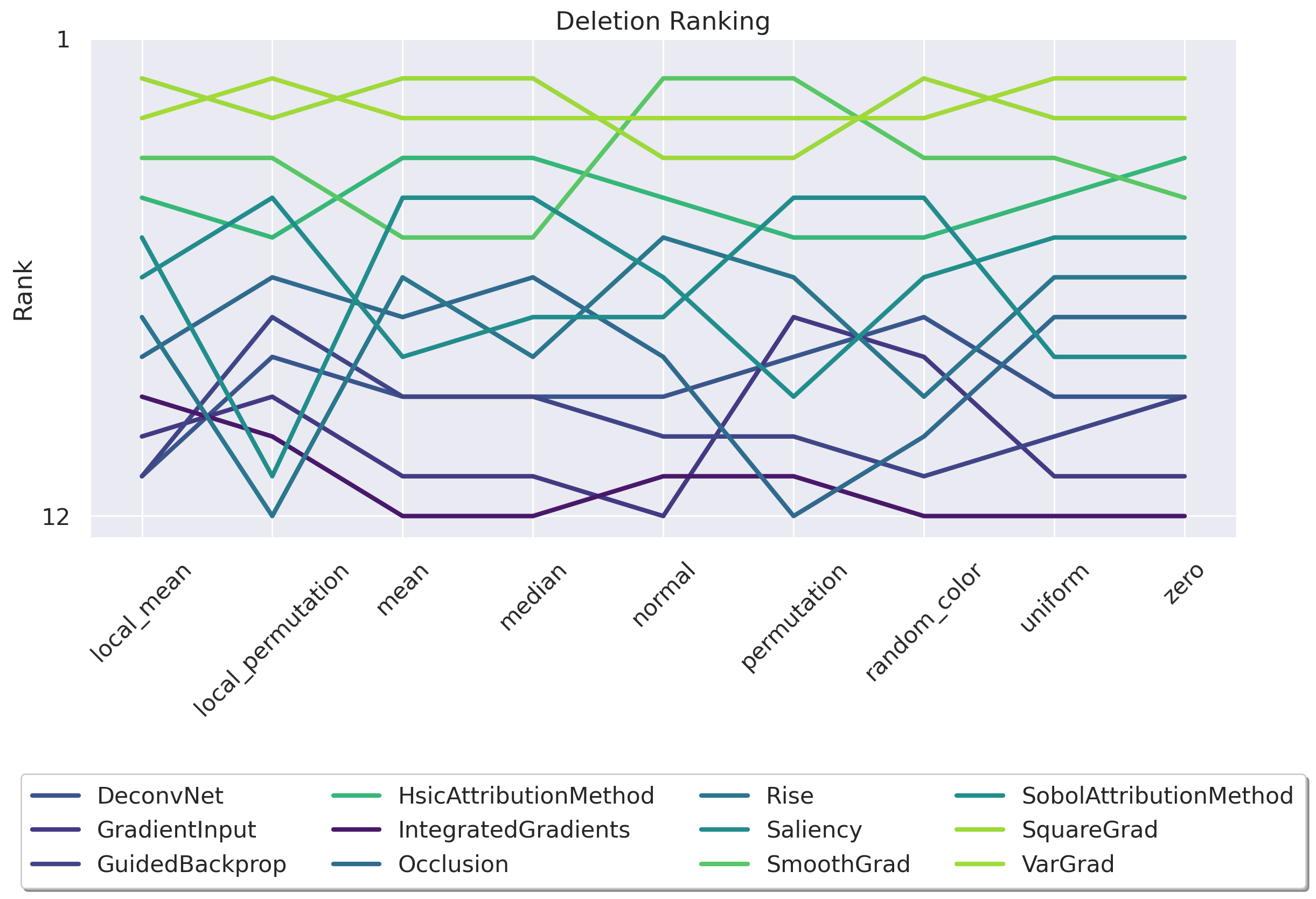}
\caption{\textbf{Ranking of attribution methods according to Insertion and Deletion for different baselines for a ViT.} We investigate how different baseline choices affect the assessment of attribution method faithfulness for a pre-trained ViT~\cite{dosovitskiy2020image} model. We evaluate Insertion and Deletion across a series of attribution methods from the literature using 9 different baselines (see Appendix~\ref{apx:baselines} for more details) on a total of 6,000 explanations. Unlike what we observe in Figure~\ref{fig:insertion-ranking}, Insertion is slightly unstable in the ViT, and Deletion is as unreliable as with the ResNet50.}
\label{fig:insertion-deletion-ranking-vit}
\end{figure}

Just like with the ResNet50, we compute the Deletion and Insertion faithfulness metrics for several attribution methods using multiple baselines and find that Deletion is too baseline-dependent to be reliable. Although slightly more stable, Insertion seems to depend much more on the baseline than on the ResNet50. This is likely due to the model's architecture and data augmentation choices during training.

\section{The Baselines}\label{apx:baselines}

In the literature, there are some baselines that come up quite often: zero, the dataset's mean, random noise, etc. In this work, we intend to be exhaustive in our findings, and hence, we devised some additional naïve baselines -- i.e., methods that one might choose to remove semantic information from images. In particular, we worked with the following baselines:

\begin{itemize}
\item \textbf{zero}: replaces the target with a value of \texttt{0.0}.
\item \textbf{random color}: substitutes the target with a randomly chosen color.
\item \textbf{uniform}: fills the target with noise sampled from a uniform distribution $\mathcal{U}[0, 1]$.
\item \textbf{normal}: replaces the target with noise drawn from a normal distribution $\mathcal{N}(0, 1)$.
\item \textbf{mean}: substitutes the target with the average color of the entire dataset.
\item \textbf{local mean}: fills the target with the average color of the specific image.
\item \textbf{median}: replaces the target with the median color of the whole dataset.
\item \textbf{permutation}: randomly shuffles the target pixels throughout the entire image.
\item \textbf{local permutation}: shuffles the target pixels within a localized area of the image.
\item \textbf{scramble magnitude}: randomly permutes the magnitude component of the image in the Fourier space.
\item \textbf{scramble phase}: randomly permutes the phase component of the image in the Fourier space.
\end{itemize}

\section{OOD curves}\label{apx:ood-curves}

As described in Section~\ref{sec:results}, we compute the OOD score at each step of the Deletion process for the image as it gets gradually more replaced by the baseline. We observe in Figure~\ref{apx:ood-curves} that some baselines can be significantly more OOD than others, which motivates our model-dependent baseline.

\begin{figure}[H]
\centering
\includegraphics[width=\textwidth]{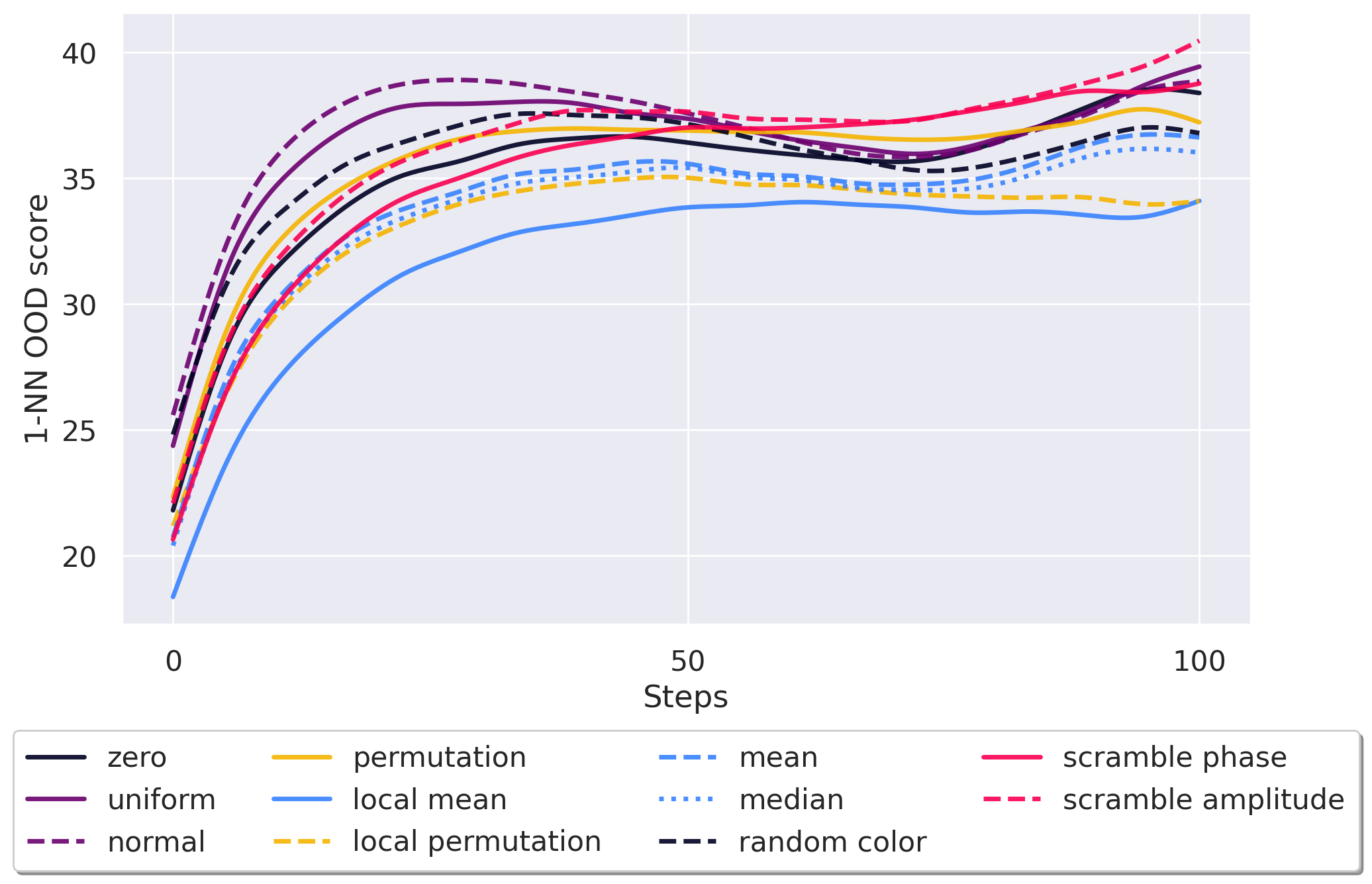}
\caption{\textbf{OOD curves during the Deletion process.} For eleven different baselines and a ResNet50V2, we estimate the OOD score of the input image as it gets progressively replaced by the baseline using the 1-NN OOD score.}
\label{fig:ood-curves}
\end{figure}

\newpage

\section{Baseline classification}\label{ap:baseline_classif}

We demonstrate that some of the baselines prevalent in the literature do not remove information efficiently. In particular, we show in Figure~\ref{fig:baseline-classif} that, when input into a ResNet50V2, the model predicts with a high softmax score for some baselines. This goes against the information removal desideratum described in Section~\ref{sec:desiderata}, which is what makes the Deletion rankings unstable.

\begin{figure}[H]
\centering
\includegraphics[width=\textwidth]{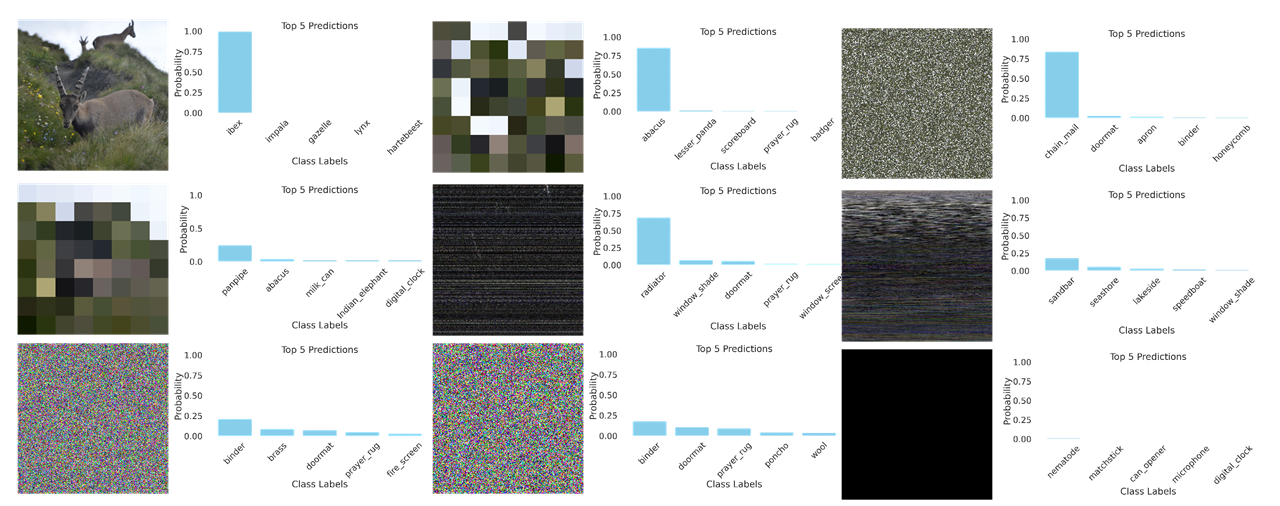}
\caption{\textbf{Classification of the baseline for a ResNet50V2.} For a ResNet50V2 pre-trained on ImageNet1k, we predict the top-5 classes for 8 different baselines, and we show that they still carry semantic information for the model.}
\label{fig:baseline-classif}
\end{figure}

\newpage

\section{Example of issues with the use of inpainting models}
\label{ap:extran_diffusion_example}

The challenge of using a diffusion model to generate the baseline lies in the difficulty of separating the impact of the inpainting model from that of the model under investigation. As illustrated in Figure~\ref{fig:diffusion_baseline_extra_example}, one attribution method generates entirely random pixels, while the other selects pixels in a more structured manner. Despite these differences, both methods ultimately reconstruct an image with the same semantic content, resulting in identical deletion scores.

\begin{figure}[H]
\centering
\includegraphics[width=0.75\textwidth]{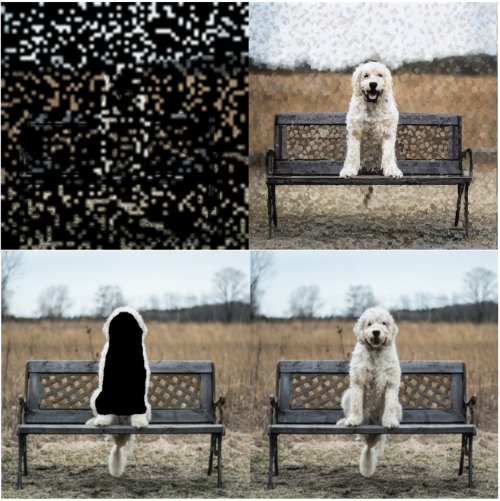}
\caption{Example with two different masks yielding very similar reconstructions that would obtain similar deletion scores.}
\label{fig:diffusion_baseline_extra_example}
\end{figure}

\end{document}